\documentclass[twoside]{article}

\usepackage[sc]{mathpazo} 
\usepackage[T1]{fontenc} 
\usepackage{microtype} 

\usepackage[hmarginratio=1:1,top=32mm,columnsep=20pt]{geometry} 
\usepackage{multicol} 
\usepackage[hang, small,labelfont=bf,up,textfont=it,up]{caption} 
\usepackage{booktabs} 
\usepackage{float} 
\usepackage{hyperref} 

\usepackage{lettrine} 
\usepackage{paralist} 

\usepackage{abstract} 

\usepackage{titlesec} 
\renewcommand\thesection{\Roman{section}}
\titleformat{\section}[block]{\large\scshape\centering}{\thesection.}{1em}{} 

\usepackage{fancyhdr} 
\usepackage{amssymb}
\usepackage{epsfig}
\usepackage{amssymb}
\usepackage{url}
\usepackage{amsmath}
\usepackage{natbib}
\usepackage{amsthm}
\usepackage{booktabs}
\usepackage[lined,boxed,linesnumbered]{algorithm2e}
\usepackage{url}
\usepackage{alltt}
\usepackage{enumerate}
\usepackage{multirow}
\usepackage{setspace}
\usepackage{bm}

\newcommand{\sx}{\mathsf{x}}
\newcommand{\rx}{\mathtt{x}}

\newcommand{\vx}{\mathbf{x}}

\newcommand{\sign}{\mathsf{sign}}

\newcommand{\nmathbf}{\bm}

\def\bfbeta   {\nmathbf \beta}

\def\boldfacefake#1{\kern-4pt
    \hbox{ \mathsurround=0pt
    \hbox to 0.4pt{$#1$\hss}\hbox to 0.4pt{$#1$\hss}\hbox {$#1$}}}

\newcommand{\be}{\begin{eqnarray}}
\newcommand{\ee}{\end{eqnarray}}
\newcommand{\ba}{\begin{eqnarray*}}
\newcommand{\ea}{\end{eqnarray*}}

\newcommand{\reals}{\mbox{\rm I\kern-.20em R}}
\newcommand{\sreals}{\mbox{\small \rm I\kern-.20em R}}

\newtheorem{theorem0}{Theorem}
\newtheorem{lemma0}{Lemma}
\newtheorem{remark0}{Remark}
\newtheorem{fact0}{Fact}
\newtheorem{example0}{Example}
\newtheorem{definition0}{Definition}
\newtheorem{corollary0}{Corollary}
\newtheorem{proposition0}{Proposition}
\newtheorem{algorithmY}{Algorithm}
\newtheorem{conjecture0}{Conjecture}

\newenvironment{theorem}{\begin{theorem0} \mbox{} }{\end{theorem0}}
\newenvironment{lemma}{\begin{lemma0} \mbox{}}{\end{lemma0}}

\newenvironment{fact}{\begin{fact0} \mbox{}}{\end{fact0}}

\newenvironment{definition}{\begin{definition0}
\mbox{}}{\end{definition0}}

\title{\vspace{-15mm}\fontsize{24pt}{10pt}\selectfont\textbf{On the Predictive Properties of \\ Binary Link Functions}} 

\author{
\large
\textsc{Necla G\"und\"uz}\\
\normalsize Department of Statistics \\ 
\normalsize Gazi \"Universitesi, Fen Fak\"ultesi, Istatistik B\"ol\"um\"u \\
\normalsize 06500 Teknikokullar, Ankara, Turkey\\
\normalsize \href{mailto:ngunduz@gazi.edu.tr}{ngunduz@gazi.edu.tr} 
\vspace{5mm}
\and
\textsc{Ernest Fokou\'e}\thanks{Corresponding Author}\\[2mm] 
\normalsize Center for Quality and Applied Statistics \\ 
\normalsize Rochester Institute of Technology \\ 
\normalsize 98 Lomb Memorial Drive, Rochester, NY 14623, USA \\ 
\normalsize \href{mailto:ernest.fokoue@rit.edu}{ernest.fokoue@rit.edu} 
\vspace{-5mm}
}
\date{}

\begin{document}

\maketitle 

\thispagestyle{fancy} 


\begin{abstract}

\noindent This paper provides a theoretical and computational justification of the long held
claim that of the similarity of the probit and logit link functions often used in binary classification.
Despite this widespread recognition of the strong similarities between these two link functions, very few (if any) researchers have dedicated
time to carry out a formal study aimed at establishing and characterizing firmly all the aspects of
the similarities and differences. This paper proposes a definition of both structural and predictive
equivalence of link functions-based binary regression models, and explores the various ways in which they
are either similar or dissimilar.
From a predictive analytics perspective, it turns out that not only are probit and logit perfectly predictively concordant,
but the other link functions like cauchit and complementary log log enjoy very high percentage of predictive equivalence.
Throughout this paper, simulated and real life examples demonstrate  all the equivalence results that we
prove theoretically.
\end{abstract}



\section{Introduction}
Given $(\vx_1,y_1),\cdots,(\vx_n,y_n),$ where $\vx_i^\top \equiv (\sx_{i1}, \cdots,\sx_{ip})$
denotes the $p$-dimensional vector of characteristics and $y_i\in\{0,1\}$ denotes the binary response variable,
binary regression seeks to model the relationship between $\vx$ and $y$ using
\begin{eqnarray}
\pi(\vx_i)=\Pr[Y_i=1|\vx_i]=F(\eta(\vx_i))
\label{eq:binreg:1}
\end{eqnarray}
where
\begin{eqnarray}
 \eta(\vx_i)= \beta_0+\beta_1 \sx_{i1} + \cdots + \beta_p \sx_{ip} =\tilde{\vx}_i^\top\bfbeta \quad i=1,\cdots,n
\label{eq:eta:1}
\end{eqnarray}
for a $(p+1)$-dimensional vector $\bfbeta = (\beta_0, \beta_1,\cdots,\beta_p)^\top$ of regression coefficients and $F(\cdot)$ is the cdf corresponding to the link functions under consideration. Specifically, the cdf $F(\cdot)$ is the inverse of the link function $g(\cdot)$, such that $\eta(\vx_i)=F^{-1}(\pi(\vx_i))=g(\pi(\vx_i))=g(\mathbb{E}(Y_i|\vx_i))$. Table \eqref{tab:links:1} provides specific definitions of the link functions considered in this paper, along with their corresponding cdfs.
\begin{table}[!h]
  \centering
  \begin{tabular}{lcc}
  \toprule
  {\sf Model}  & {\sf Link function} & {\sf cdf} \\ \hline
  \vspace{0.25cm}
  {\sf Probit}  & $\Phi^{-1}(v)$ & $\Phi(u)$ \\
  \vspace{0.25cm}
  {\sf  Compit}  & $\log[-\log(1-v)]$ & $1-e^{-e^u}$ \\
  \vspace{0.25cm}
  {\sf  Cauchit} & $\tan\Big[\pi v-\frac{\pi}{2}\Big]$ & $\frac{1}{\pi}\Big[\tan^{-1}(u)+\frac{\pi}{2}\Big]$ \\
  {\sf  Logit}   & $\log\Big[\frac{v}{1-v}\Big]$ & $\Lambda(u)=\frac{1}{1+e^{-u}}$ \\ \bottomrule
  \end{tabular}
  \caption{\em Link functions along with corresponding cdfs}
  \label{tab:links:1}
\end{table}

\noindent The above link functions have been used extensively in a wide variety of applications in fields
as diverse as medicine, engineering, economics, psychology, education just to name a few.
The logit link function for which
\begin{eqnarray}
\pi(\vx_i)=\Pr[Y_i=1|\vx_i]= \Lambda(\eta(\vx_i)) = \frac{1}{1+e^{-\eta(\vx_i)}}
\label{eq:logit:1}
\end{eqnarray}
is the most commonly used of all of them, probably because it provides a nice
interpretation of the regression coefficients in terms of the ratio of the odds. The popularity of the logit link also comes from
its computational convenience in the sense that its model formulation yields simpler maximum likelihood equations and faster convergence.
In fact, the literature on both the theory and applications based on the logistic distribution is so vast it would be unthinkable to reference even a fraction of it. Some recent authors like \cite{Zelterman:89:1}, \cite{Schumacher:96:1}, \cite{Nadarajah:04:1}, \cite{Lin:08:1} and \cite{Nassar:12:1} provide extensive studies on the characteristics of generalized logistic distributions, somehow answering the ever increasing interest
in the logistic family of distributions. Indeed, applications abound that make use of both the standard logistic regression model and the so-called generalized logistic regression model, as can be seen in \cite{Hout:07:1} and  \cite{Tamura:13:1}.
The probit link, for which
\begin{eqnarray}
\pi(\vx_i)&=&\Pr[Y_i=1|\vx_i]=F(\eta(\vx_i)) \nonumber \\
&=& \Phi(\eta(\vx_i)) = \int_{-\infty}^{\eta(\vx_i)}{\frac{1}{\sqrt{2\pi}}e^{-\frac{1}{2}z^2}dz}
\label{eq:probit:1}
\end{eqnarray}
is the second most commonly used of all the link functions,
with Bayesian researchers seemingly topping the charts in its use. See \cite{Basu:00:1}, \cite{Csato:Fokoue:2000}, \cite{Sounak:09:1} for a few examples of probit use in binary classification in the Bayesian setting. \cite{Armagan:11:1} is just another
one of the references pointing to the use of the probit link function in the statistical data mining and
machine learning communities.

\noindent In the presence of some many possible choices of link functions, the natural question to ask is: how does one go about choosing the
right/suitable/appropriate link function for the problem at hand? Most experts and non-experts alike who deal with binary classification tend to almost automatically choose the logit link, to the point that it - the logit link - has almost been attributed a transcendental place. From experience, experimentation and mathematical proof,
it is our view, a view shared by \cite{Feller} and \cite{Feller:1940:1}, that all these link function are equivalent, both structurally and predictively. Indeed, our conjectured equivalence of binary regression link functions is strongly supported by William Feller in his vehement criticism of the overuse of the logit link function and a tendency to give it a place above the rest of existing link functions. In \cite{Feller}'s own words: {\it  An unbelievably huge literature tried to establish a transcendental "law of logistic growth"; measured in appropriate units, practically all growth processes were supposed to be represented by a function of the form \eqref{eq:logit:1} with $t$ representing time. Lengthy tables, complete with chi-square tests, supported this thesis for human populations, for bacterial colonies, development of railroads, etc. Both height and weight of plants and animals were found to follow the logistic law even though it is theoretically clear that these two variables cannot be subject to the same distribution. Laboratory experiments on bacteria showed that not even systematic disturbances can produce other results. Population theory relied on logistic extrapolations (even though they were demonstrably unreliable). The only trouble with the theory is that not only the logistic distribution but also the normal, the Cauchy, and other distributions can be fitted to the same material with the same or better goodness of fit. In this competition the logistic distribution plays no distinguished role whatever; most contradictory theoretical models can be supported by the same observational material}.

\vspace{0.5cm}

\noindent As a matter of fact, it's obvious from the plot of their densities for instance that the probit and logit are virtually identical, almost superposed one on top of the other. It is therefore not surprising that one would empirically
notice virtually no difference when the two are compared on the same binary regression task.
Despite this apparent indistinguishability due to many of their similarities,
it is fair to recognize that the two functions different, at least by definition and by their very algebra. \cite{Chambers:67:1} argue in their paper that probit and logit will yield different results in
the multivariate context. Their work is a rarety in a context where most researchers seem to have settled
comfortably with the acceptance of the fact that the two links are essentially the same from
a utility perspective. For such researchers,  using one over the other is determined solely by mathematical convenience and a matter of taste. We demonstrate both theoretically and computationally
that they all predictively equivalent in the univariate case, but we also provide a characterization
of the conditions under which they tend to differ in the multivariate context.

\begin{figure}[!h]
\centering
\epsfig{figure=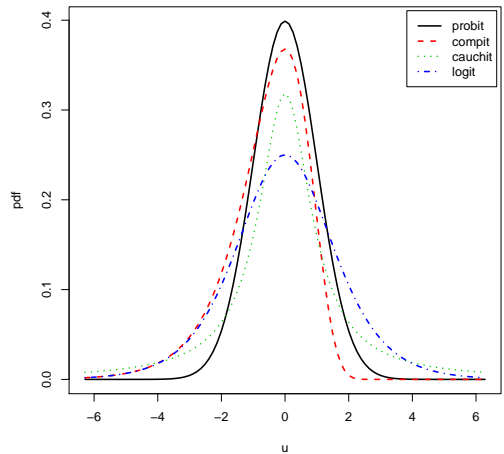, height=7cm, width=7cm} 
\caption{{\it Densities corresponding to the link functions. The similarities are around the center of the distributions. Differences can be seen at the tails}}
\label{fig:links:1a}
\end{figure}

\begin{figure}[!h]
\centering
\epsfig{figure=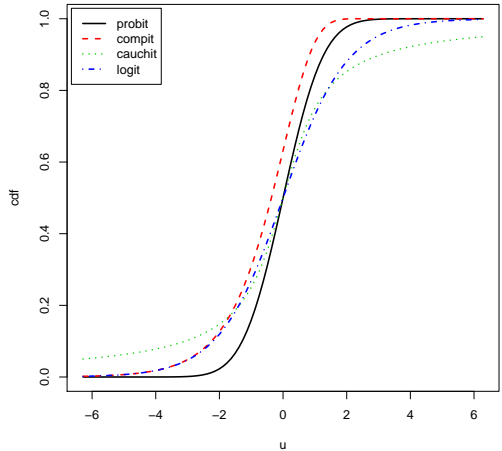, height=7cm, width=7cm}
\caption{{\it Cdfs corresponding to the link functions. The similarities are around the center of the distributions. Differences can be seen at the tails}}
\label{fig:links:1b}
\end{figure}

\noindent Throughout this work, we perform model comparison and model selection using both Akaike Information Criterion (AIC) and
Bayesian Information Criterion (BIC). Taking the view that the ability of an estimator to generalize well over the whole population, provides the best measure of it ultimate utility, we provide extensive comparisons of the performances of each link functions based on their corresponding test error. In the present work, we perform a large number of simulations in various dimensions using both  artificial and real life data. Our results persistently reveal the performance indistinguishability of the links in univariate settings, but some sharp differences begin to appear as the dimension of the input space (number of variables measured) increased.


\noindent The rest of this paper is organized as follows: section $2$ presents some general definitions, namely our meaning of the terms predictive equivalence and structural equivalence, along with some computational demonstrations on simulated and real life data. This section also clearly describes our approach to demonstrating/verifying our claimed results. We show in this section, that for low to moderate dimensional spaces, goodness of fit and predictive performance measures reveal the equivalence between probit and logit. Section $3$ provides our formal proof of the equivalence of probit and logit. Section $4$ reveals that there might be some differences in performance when the input space becomes very large. Our demonstration in this section in based on the famous AT\&T $57$-dimensional Email Spam Data set. Section $5$ provides a conclusion and a discussion, along with insights into extensions of the present work.

\section{Definitions, Methodology and Verification}
Throughout this work, we consider comparing models both on the merits of goodness of fit, and predictive performance.
With that in mind, we can then define equivalence both from a goodness of fit perspective and also from a predictive
optimality perspective.From a predictive analytics perspective for instance, an important question to ask is: given a randomly selected
vector $\vx$,  what is the probability that the prediction made by probit will differ from the one made by logit?
In other words, how often do the probit and logit link functions yield difference predictions? This is particularly important
in predictive analytics in the data mining and machine learning where the nonparametric nature of most models forces the experimenter
to focus on the utility of the estimator rather than its form. We respond to this need by defining what we call the $100(1-\alpha)\%$
predictive equivalence.

\subsection{Basic definitions and results}
\label{subsection:definitions}
\begin{definition}(Binary classifier)
Given an input space $\mathcal{X}$ and a binary response space $\mathcal{Y}=\{0,1\}$, we define a (binary) classifier $h$ to be
a function that maps elements of $\mathcal{X}$ to $\{0,1\}$, or more specifically
$$
\begin{array}{l}
h:\,\,\mathcal{X}  \rightarrow \{0,1\} \\
\qquad\vx \mapsto h(\vx)
\end{array}
$$
In the generalized linear model (GLM) framework, given a link function with corresponding cdf $F(\cdot)$, a binary classifier $h$ under the majority rule takes the
form
$$
{h}(\vx) = \frac{1}{2}\left\{1+\sign\left({\pi}(\vx)-\frac{1}{2}\right)\right\},
$$
where ${\pi}(\vx)=\Pr[Y=1|\vx]=F(\eta(\vx))$ and $\eta(\vx)$ is the linear component. For instance, the logit binary classifier
is given by
$$
h_{logit}(\vx) = \frac{1}{2}\left\{1+\sign\left({\Lambda}(\eta(\vx))-\frac{1}{2}\right)\right\},
$$
and the the probit binary classifier is given by
$$
h_{probit}(\vx) = \frac{1}{2}\left\{1+\sign\left({\Phi}(\eta(\vx))-\frac{1}{2}\right)\right\},
$$
where $\Lambda(\cdot)$ and $\Phi(\cdot)$ are as defined in Table \eqref{tab:links:1}.
\end{definition}

\noindent We shall measure the predictive performance of a classifier $h$ by choosing a loss function $\ell(\cdot,\cdot)$ and
then computing the expected loss (also known as risk functional) $R(h)$ as follows:
$$
R(h) = \mathbb{E}[\ell(Y,h(X))] = \int_{\mathcal{X}\times \mathcal{Y}}{\ell(y,h(\vx)) p(\vx,y)d\vx d y}.
$$
Under the zero-one loss function $\ell(Y,h(X)) = \mathbf{1}_{\{Y \neq h(X)\}}$, the risk functional $R(h)$
is the misclassification rate, more specifically
\begin{eqnarray*}
R(h) &=& \mathbb{E}[\ell(Y,h(X))] \\
&=& \int_{\mathcal{X}\times \mathcal{Y}}{\ell(y,h(\vx)) p(\vx,y)d\vx d y} = \Pr[Y\neq h(X)].
\end{eqnarray*}
In practice, $R(h)$ cannot be computed in closed-form because the distribution of $(X,Y)$ is unknown. We shall therefore use
the so-called the average test error or average empirical prediction error as our predictive performance measure to
compare classifiers.

\begin{definition}(Average Test Error)
Given a sample $\{(\vx_i, y_i), i=1,\cdots,n\}$, we randomly form a training set $\{(\vx_i^{(\tt tr)}, y_i^{(\tt tr)}), i=1,\cdots,n_{\tt tr}\}$ and a test set $\{(\vx_i^{(\tt te)}, y_i^{(\tt te)}), i=1,\cdots,n_{\tt te}\}$. We typically
run $R=1000$ replications of this split, with $2/3$ of the data allocated to the training set and $1/3$ to the test set. The test error here under the symmetric zero-one loss is given by
\begin{eqnarray*}
\hat{R}_{\tt \tiny test}(h) = {\tt TE}(h) &=&
\frac{1}{n_{\tt te}}\sum_{i=1}^{n_{\tt te}}{\mathbf{1}_{\{y_i^{(\tt te)} \neq h(\vx_i^{(\tt te)})\}}} \\
&=& \frac{\#{\{y_i^{(\tt te)} \neq h(\vx_i^{(\tt te)})\}}}{n_{\tt te}},
\end{eqnarray*}
from which the average test error of $h$ over $R$ random splits of the data is given by
$$
{\tt ATE}(h) = \frac{1}{R}\sum_{r=1}^R{{\tt TE}_r(h)},
$$
where ${\tt TE}_r(h)$ is the test error yielded by $h$ on the $r$th split of the data.
\end{definition}

\begin{definition}(Predictively concordant classifiers)
Let $h_1$ and $h_2$ be two classifiers defined on the same $p$-dimensional input space $\mathcal{X}$. We shall say that $h_1$ and $h_2$
are $100(1-\alpha)\%$ predictively concordant if $\forall X \in \mathcal{X}$ drawn according to the density $p_X(\vx)$,
$$
\Pr\Big[{h}_{1}(X) \neq {h}_{2}(X) \Big] = \alpha.
$$
In other words, $h_1$ and $h_2$ are $100(1-\alpha)\%$  predictively concordant if the probability of disagreement between the two classifiers is $\alpha$. When $\alpha=0$, we say that $h_1$ and $h_2$ are perfectly predictively concordant.
\end{definition}

\begin{definition}(Predictively equivalent classifiers)
Let $h_1$ and $h_2$ be two classifiers defined on the same $p$-dimensional input space $\mathcal{X}$. We shall say that $h_1$ and $h_2$
are predictively equivalent if the difference between their average test errors is negligible, i.e.,
${\tt ATE}(h_1) \approx {\tt ATE}(h_2)$.
\end{definition}


\begin{lemma}
\label{lemma:goldenratio}
If $X \sim {\tt Logistic}(0,1)$, and $Y = \sqrt{\frac{\pi}{8}}X$, then $Y \overset{\tiny approx}{\sim } N(0,1)$.
\end{lemma}
\noindent \textit{Demonstration:}
Figure (\ref{fig:probit:scaled}) below shows that the scaled version of the logistic cdf lines up almost perfectly with the standard normal.
\begin{figure}[!htpb]
\centering
\epsfig{figure=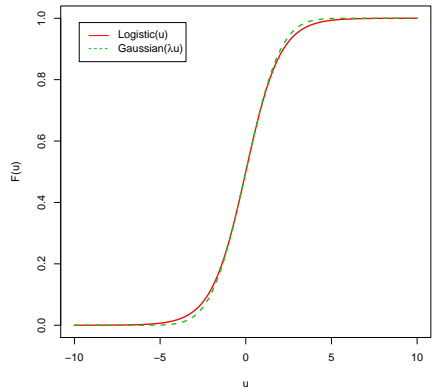, width=6cm, height=6cm}
\caption{Scaled version of the logistic CDF superposed on the standard normal CDF.}
\label{fig:probit:scaled}
\end{figure}

\begin{lemma}
\label{lemma:probit}
Let $\Phi(\cdot)$ denote the standard normal cdf.  Then 
$$
\sign\left({\Phi}(z)-\frac{1}{2}\right) = \sign\left({\Phi}\left(\sqrt{\frac{\pi}{8}} z\right)-\frac{1}{2}\right).
$$
\end{lemma}

\begin{theorem}
The probit and logit link functions are perfectly predictively concordant. Specifically, given an input space
$\mathcal{X}$ and a density $p_X(\vx)$ on $\mathcal{X}$,
$$
\Pr\Big[{h}_{\tt logit}(X) \neq {h}_{\tt probit}(X) \Big] = 0,
$$
for all $X \in \mathcal{X}$ drawn according to $p_X(\vx)$.
\end{theorem}

\begin{proof}
For a given $X$, Let $E$ be the event
$$
E = \left\{\sign\left({\Lambda}(\eta(X))-\frac{1}{2}\right) \neq \sign\left({\Phi}(\eta(X))-\frac{1}{2}\right)\right\}
$$
we must show that $\delta = \Pr[E]=0$.
Based on Lemma \eqref{lemma:goldenratio}, we can write $E$ as
$$
E = \left\{\sign\left({\Phi}\left(\sqrt{\frac{\pi}{8}}\eta(X)\right)-\frac{1}{2}\right) \neq \sign\left({\Phi}(\eta(X))-\frac{1}{2}\right)\right\}.
$$
Then $\delta = \Pr[E]=0$.
Thanks to Lemma \eqref{lemma:probit}, it is straightforward to see that $\delta=0$.
\end{proof}

\begin{definition}
Let $M_1$ and $M_2$ be two binary regression models based on two different link functions defined on the same $p$-dimensional input space.
We shall say that $M_1$ and $M_2$ are structurally equivalent if there exists a nonzero real constant $\lambda \in \mathbb{R}^*$ such that $\beta_j^{(\tt M_1)} \approx \lambda \beta_j^{(\tt M_2)}$ for all $j=1,\cdots,p$. In other words, the parameters of $M_1$ are just a scaled version of the parameters of $M_2$, so that knowing the parameters of $M_1$ is sufficient
to completely determine the parameters of $M_2$, and vice-versa.
\end{definition}

\begin{theorem}
The logit and probit models are structurally equivalent.
\end{theorem}

\begin{proof}
Thanks to Lemma \eqref{lemma:goldenratio}, we can write
\begin{eqnarray*}
\Lambda(\vx^\top \bfbeta^{(\tt logit)}) &\approx& \Phi\left(\sqrt{\frac{\pi}{8}} \vx^\top \bfbeta^{(\tt logit)}\right) \\
&=& \Phi\left(\vx^\top \sqrt{\frac{\pi}{8}}\bfbeta^{(\tt logit)}\right)=\Phi(\vx^\top \bfbeta^{(\tt probit)}),
\end{eqnarray*}
where
$$
\bfbeta^{(\tt probit)} \approx \sqrt{\frac{\pi}{8}} \bfbeta^{(\tt logit)}.
$$
We have therefore found a nonzero real constant $\lambda=\sqrt{\frac{\pi}{8}}$ such that $\bfbeta^{(\tt probit)} \approx \lambda \bfbeta^{(\tt logit)}$.
\end{proof}

\subsection{Computational Verification via Simulation}
To get deeper into how strongly related the probit and logit models are, we now seek to estimate via simulation, the constant coefficient that relates their parameter estimates. Indeed, we conjecture that $\hat{\bfbeta}^{(\tt logit)}$ and $\hat{\bfbeta}^{(\tt probit)}$ are linearly related via the regression equation
$$
\hat{\bfbeta}^{(\tt probit)} = \tau + \theta \hat{\bfbeta}^{(\tt logit)} + \nu,
$$
where $\tau$ is the intercept and $\nu$ is the noise term. To estimate one instance of $\theta$, we generate $M$ random replications of the dataset, and for each replication we estimate a copy of $\hat{\bfbeta}$, and with it we also compute an estimate of $\rho = {\tt cor}(\hat{\bfbeta}^{(\tt probit)}, \hat{\bfbeta}^{(\tt logit)})$ the correlation coefficient between $\hat{\bfbeta}^{(probit)}$ and $\hat{\bfbeta}^{(logit)}$.
By repeating the estimation $R$ times, we gather data to determine the central tendency of $\theta$ and the corresponding correlation.

\vspace{0.2cm}

\hrule
\begin{itemize}
\item[] {\tt For r = 1 to R}
 \begin{itemize}
    \item[] {\tt For s = 1 to S}
    \begin{itemize}
    \item {\tt Generate a replicate of the random sample of $\{(\vx_i,y_i), i=1,\cdots,n\}$}
    \item {\tt Estimate the logit and probit model coefficients} $\hat{\bfbeta}_s^{(logit)}$ and $\hat{\bfbeta}_s^{(probit)}$
    \end{itemize}
    \item[] {\tt End}
    \item {\tt Store the simulated data $\mathcal{D}^{(r)}=\{(\hat{\bfbeta}_s^{(logit)}, \hat{\bfbeta}_s^{(probit)}), \,\, s=1,\cdots,S\}$}
    \item {\tt Fit $\mathcal{M}^{(r)}$, the regression  model $\hat{\bfbeta}_s^{(probit)} = \tau + \theta\hat{\bfbeta}_s^{(logit)} + \nu_s$ using $\mathcal{D}^{(r)}$}
    \item {\tt Extract the coefficient $\hat{\theta}^{(r)}$ from $\mathcal{M}^{(r)}$}
    \item {\tt Compute $\hat{\rho}^{(r)}$ estimate of correlation between $\hat{\bfbeta}^{(probit)}$ and $\hat{\bfbeta}^{(logit)}$}
 \end{itemize}
\item[] {\tt Collect $\{\hat{\theta}^{(r)}\, \texttt{and} \, \hat{\rho}^{(r)}, \,\, r=1,\cdots,R\}$, then compute relevant statistics.}
\end{itemize}

\hrule

\vspace{0.2cm}

\noindent {\tt Example 1:} We consider a random sample of $n=199$ observations $\{(\vx_i, y_i), i=1,\cdots,n\}$ where the $\vx_i$ are equally spaced points in an interval $[a,b]$, that is, $\vx_i = a + \left(\frac{b-a}{n-1}\right)(i-1)$, and $y_i$ are drawn from one of the binary regression models. For instance, we set the domain of $\vx_i$ to $[a,b]=[0,1]$ and generate the $Y_i$'s from a Cauchit model with slope $1/2$ and intercept $0$, i.e., $Y_i \overset{iid}{\sim} {\tt Bernoulli}(\pi(\vx_i))$, with
$$
\Pr[Y_i = 1 | \vx_i] = \pi(\vx_i) = \frac{1}{\pi}\left[\tan^{-1}\left(\frac{1}{2} \vx_i\right)+\frac{\pi}{2}\right],
$$
Using $R=99$ replications each running $S=199$ random samples, we obtain the following results, see Fig (\ref{fig:sim:1}). The most striking finding here is that the estimated coefficient of determination is roughly equal to $1$, indicating that the knowledge of logit coefficient almost entirely helps determine the value
of the probit coefficient. Hence our claim of structural equivalence between probit and logit. The value of the slope $\theta$ appears to be in the neighborhood of $0.6$.

\begin{figure}[!htpb]
\centering
\epsfig{figure=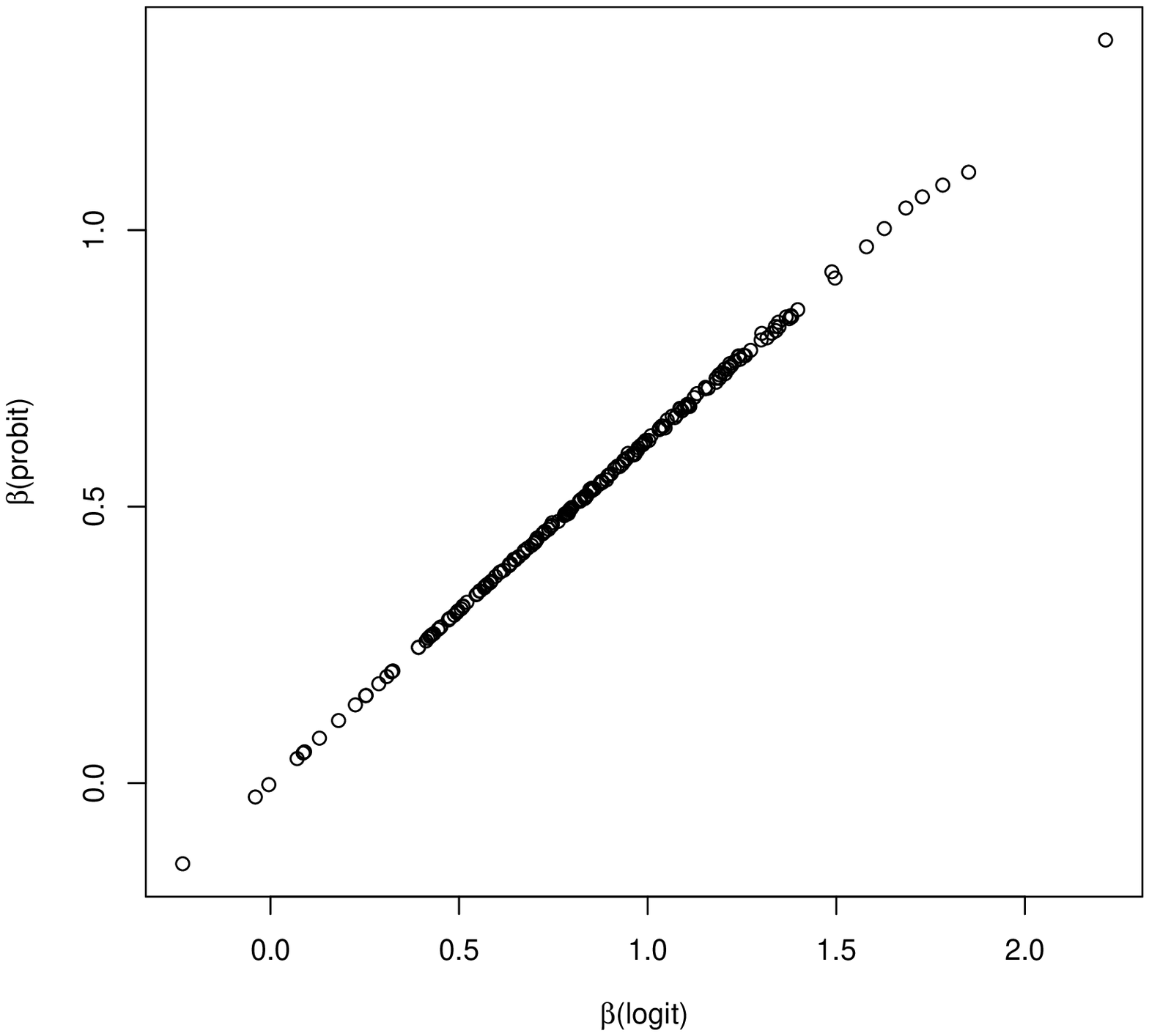, height=4cm, width=4cm}
\epsfig{figure=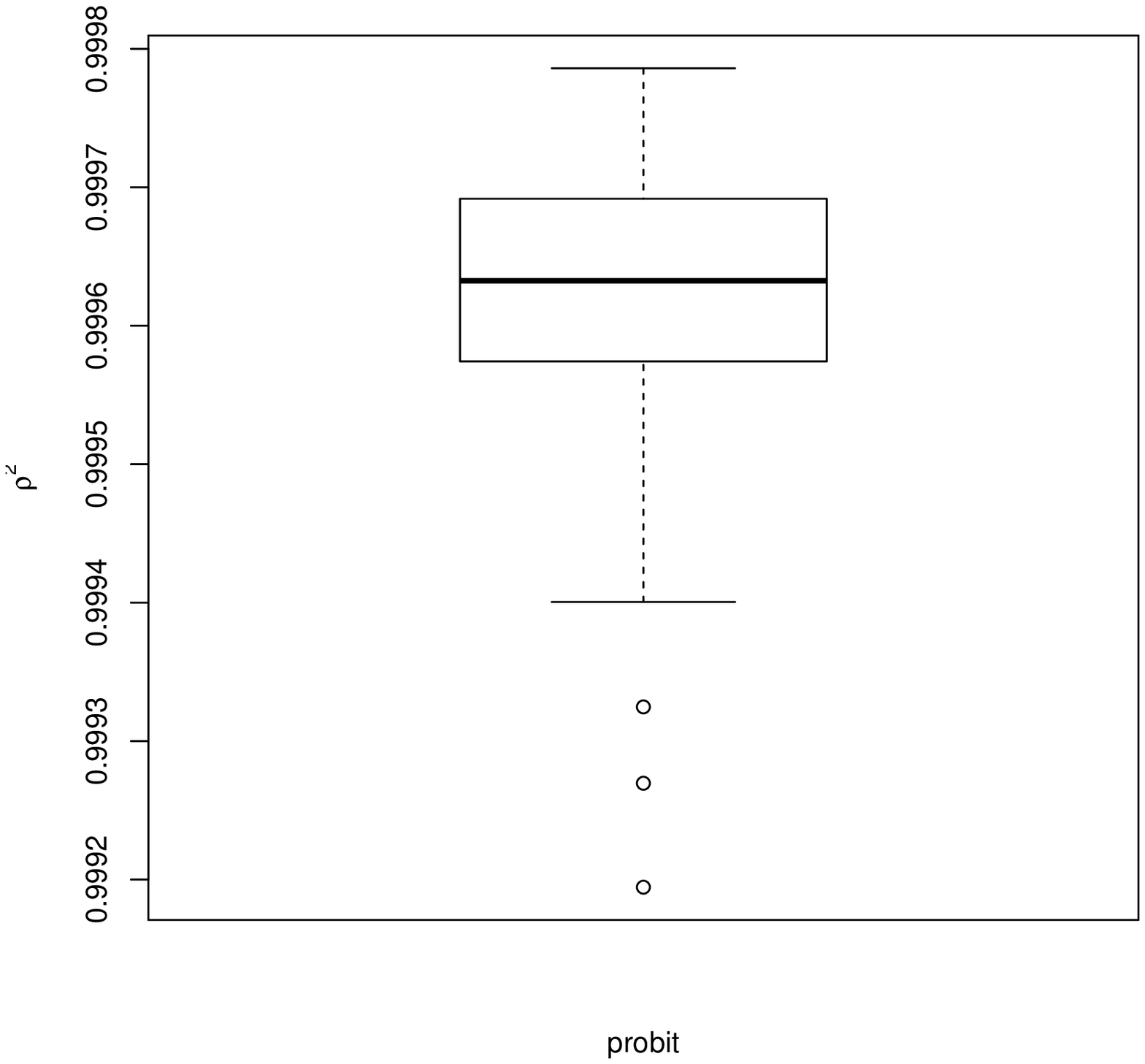, height=4cm, width=4cm}
\epsfig{figure=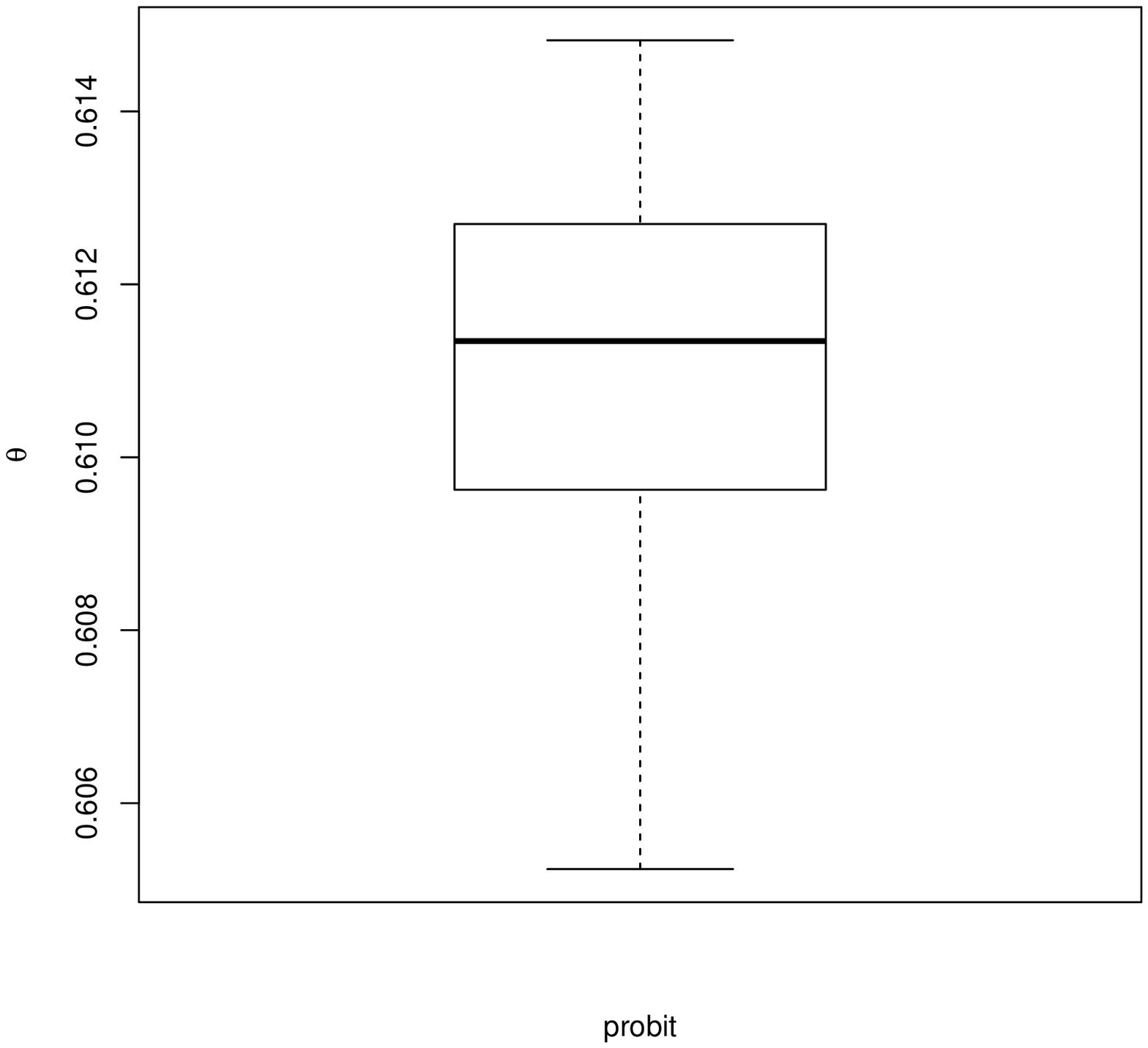, height=4cm, width=4cm}
\caption{\it (Left) Scatterplot of $\hat{\bfbeta}^{(probit)}$ against $\hat{\bfbeta}^{(logit)}$ based on the $R$ replications generated; (Center) Boxplot of the $R$ replications of the estimate of the coefficient of determination between $\hat{\bfbeta}^{(probit)}$ and $\hat{\bfbeta}^{(logit)}$; (Right)
Boxplot of the $R$ replications of the estimate of the slope $\theta$} 
\label{fig:sim:1}
\end{figure}

\noindent {\tt Example 2:} We now consider the famous Pima Indian Diabetes dataset, and obtain parameter estimates under both the logit and the probit models.
The dataset is $7$-dimensional, with $\rx_1={\tt npreg}$, $\rx_2={\tt glu}$, $\rx_3={\tt bp}$, $\rx_4={\tt skin}$, $\rx_5={\tt bmi}$, $\rx_6={\tt ped}$
and $\rx_7={\tt age}$. Under the logit model, the probability that patient $i$ has diabetes given its characteristics $\vx_i$ is given by
$$
\Pr[{\tt Diabetes}_i = 1 | \vx_i] = \pi(\vx_i) = \frac{1}{1+e^{-\eta(\vx_i)}},
$$
where
\begin{eqnarray*}
\eta(\vx_i) &=& \beta_0 + \beta_1 {\tt npreg} + \beta_2 {\tt glu} + \beta_3 {\tt bp}
+ \beta_4{\tt skin} \\
&+&\beta_5 {\tt bmi} +\beta_6 {\tt ped} + \beta_7 {\tt age}.
\end{eqnarray*}
We obtain the parameter estimates using {\tt R}, and we display in the following table their values.

\begin{table}[!h]
  \centering
  \begin{tabular}{lccccccc}
  \toprule
  {\tt Model}  & {\tt npreg} & {\tt glu} & {\tt bp} & {\tt skin}  \\ \hline
  \vspace{0.25cm}
  {\tt Probit}  & $0.0592$ & $0.0192$ & $-0.0024$ & $-0.0017$ \\
  \vspace{0.25cm}
  {\tt  Logit}  & $0.1031$ & $0.0321$ & $-0.0047$ & $-0.0019$ \\ \hline
  {\tt  Ratio}  & $0.57434$ & $0.5987$ & $0.5181$ & $0.9073$ \\
  \bottomrule
  \end{tabular}
  \caption{\em Parameter estimates under probit and logit for the Pima Indian Diabetes Data Set}
  \label{tab:pima:1a}
\end{table}

\begin{table}[!h]
  \centering
  \begin{tabular}{lccc}
  \toprule
  {\tt Model}  & {\tt bmi} & {\tt ped} & {\tt age} \\ \hline
  \vspace{0.25cm}
  {\tt Probit}  &  $0.0505$ & $ 1.0682$ & $0.0249$ \\
  \vspace{0.25cm}
  {\tt  Logit}  &  $0.0836$ & $1.8204$ & $0.0411$ \\ \hline
  {\tt  Ratio}  &  $0.6044$ & $0.5868$ & $0.6064$ \\
  \bottomrule
  \end{tabular}
  \caption{\em Parameter estimates under probit and logit for the Pima Indian Diabetes Data Set}
  \label{tab:pima:1b}
\end{table}

\noindent  As can be seen in the above Table (\ref{tab:pima:1}), the ratio of the probit coefficient over the logit coefficient is still a number around $0.6$ for almost all the parameter. Indeed, the relationship
$$
\hat{\beta}_j^{(probit)} \simeq \tau + 0.6 \hat{\beta}_j^{(logit)} + \nu
$$
appears to still hold true. The deviation from that pattern observed in variable {\tt skin} is probably due to the extreme outlier in its distribution. It is important to note that although our theoretical justification was built under the simplified setting of a univariate model with no intercept, the relationship uncovered still holds true in a complete multivariate setting, with each predictor variable obeying the same relationship.

\noindent  {\tt Example 3:} We also consider the benchmark Crabs Leptograpsus dataset, and obtain parameter estimates under both the logit and the probit models.
The dataset is $5$-dimensional, with $\rx_1={\tt FL}$, $\rx_2={\tt RW}$, $\rx_3={\tt CL}$, $\rx_4={\tt CW}$ and $\rx_5={\tt BD}$.
Under the logit model, the probability that the sex of crab $i$ is {\tt male} given its characteristics $\vx_i$ is given by
$$
\Pr[{\tt sex}_i = 1 | \vx_i] = \pi(\vx_i) = \frac{1}{1+e^{-\eta(\vx_i)}},
$$
where
$$
\eta(\vx_i) = \beta_0 + \beta_1 {\tt FL} + \beta_2 {\tt RW} + \beta_3 {\tt CL} + \beta_4{\tt CW}
+\beta_5 {\tt BD}.
$$
We obtain the parameter estimates using {\tt R}, and we display in the following table their values.

\begin{table}[!h]
  \centering
  \begin{tabular}{lccccc}
  \toprule
  {\tt Model}  & {\tt FL} & {\tt RW} & {\tt CL} &  {\tt BD}\\ \hline
  \vspace{0.25cm}
  {\tt Probit}  & $-3.5572$ & $-11.4801$ & $5.5364$ & $1.6651$ \\
  \vspace{0.25cm}
  {\tt  Logit}  & $-6.1769$ & $-19.9569$ & $9.6643$ & $2.8927$  \\ \hline
  {\tt  Ratio}  & $0.5758$ & $0.5752$ & $0.5728$ & $0.5756$ \\
  \bottomrule
  \end{tabular}
  \caption{\em Parameter estimates under probit and logit for the Crabs Leptograpsus Data Set}
  \label{tab:crabs:1}
\end{table}

\noindent As can be seen in the above Table (\ref{tab:crabs:1}), the estimate $\hat{\theta}$ of the ratio $\theta$ of the probit coefficient over the logit coefficient is still a number around $0.6$ for almots all the parameter. Indeed, the relationship
$$
\hat{\beta}_j^{(probit)} \simeq \tau + 0.6 \hat{\beta}_j^{(logit)} + \nu
$$
appears to still hold true.  It is important to note that although our theoretical justification was built under the simplified setting of a univariate model with no intercept, the relationship uncovered still holds true in a complete multivariate setting, with each predictor variable obeying the same relationship.

\begin{fact}
As can be seen from the examples above, the value of $\hat{\theta}$ lies in the neighborhood of $0.6$, regardless of the task under consideration.
This supports and confirms our conjecture  that there is a fixed linear relationship between probit coefficients and logit coefficients to the point that knowing one implies knowing the other. Hence, the two models are structurally equivalent. In a sense, wherever logistic regression has been used successfully, probit regression will do just as a job. This result confirms what was already noticed and strongly expressed by
\cite{Feller} (pp 52-53).
\end{fact}

\subsection{Likelihood-based verification of structural equivalence}
\label{subsection:likelihood}
In the proofs presented earlier, we focused on the parameters and never mentioned their estimates. We now provide a likelihood based verification of the structural equivalence of probit and logit. Without loss of generality, we shall focus on the univariate case where the underlying linear model  does not have the
intercept $\beta_0$, so that $\eta(\vx_i) = \bfbeta \vx_i$. With $\vx_i$ denoting the  predictor variable for the $i$th observation, we have the probability model $\Pr[Y_i = 1 | \vx_i] = \pi(\vx_i) = F(\eta(\vx_i)) = F(\bfbeta \vx_i)$. Let $\hat{\bfbeta}^{(\tt logit)}$ and $\hat{\bfbeta}^{(\tt probit)}$ denote the estimates of $\bfbeta$ for the logit and the probit link functions respectively. Our first verification of the equivalence of the above link functions consists of showing that
$\hat{\bfbeta}^{(\tt logit)}$ and $\hat{\bfbeta}^{(\tt probit)}$ are linearly related through $\hat{\bfbeta}^{(\tt probit)} =\tau + \theta \hat{\bfbeta}^{(\tt logit)} + \nu$, with a coefficient of determination very close to $1$ and a slope $\theta$ that remains fixed regardless of the task at hand.  We derive the approximate estimates of $\theta$ theoretically using Taylor series expansion, but we also confirm their values computationally by simulation.

\begin{theorem} Consider an i.i.d sample $(\vx_1,y_1),(\vx_2,y_2),\cdots,(\vx_n,y_n)$ where  $\vx_i\in \mathbb{R}$ is a real-valued predictor variable, and  $y_i \in \{0,1\}$ is the corresponding binary response. First consider fitting the probit model
$\Pr[Y_i = 1 | \vx_i] = \pi(\vx_i) = \Phi(\bfbeta \vx_i)$ to the data, and let $\hat{\bfbeta}^{(\tt probit)}$ denote the corresponding estimate of $\bfbeta$. Then consider fitting the logit model and $\Pr[Y_i = 1 | \vx_i] = \pi(\vx_i) = 1/(1+\exp(-\bfbeta \vx_i))$ to the data, and let $\hat{\bfbeta}^{(\tt logit)}$ denote the corresponding estimate of $\bfbeta$. Then,
$$
\hat{\bfbeta}^{(\tt probit)} \simeq 0.625 \hat{\bfbeta}^{(\tt logit)}.
$$
\end{theorem}

\begin{proof} Given an i.i.d sample $(\vx_1,y_1),(\vx_2,y_2),\cdots,(\vx_n,y_n)$ and the model $\Pr[Y_i = 1 | \vx_i] = \pi(\vx_i)$,  the loglikelihood for $\bfbeta$ is given by
\begin{eqnarray*}
\ell(\bfbeta) &=& \log L(\bfbeta) \\
&=& \sum_{i=1}^n\Big\{{y_i}\log\pi(\vx_i) + ({1-y_i})\log(1-\pi(\vx_i))\Big\}.
\end{eqnarray*}
Under the logit link function, we have $\pi(\vx_i) = 1/(1+e^{-\bfbeta \vx_i})$.
Now, using a Taylor series expansion around zero for the two most important parts of the loglikelihood function, we get
\begin{eqnarray*}
\frac{\partial {\log(\pi(\vx_i))}}{\partial \bfbeta} =
\frac{\vx_i}{2} - \frac{\vx_i^2}{4}\bfbeta +\frac{\vx_i^4}{48}\bfbeta^3-\frac{\vx_i^6}{480}\bfbeta^5,
\end{eqnarray*}
and
\begin{eqnarray*}
\frac{\partial {\log(1-\pi(\vx_i))}}{\partial \bfbeta}
= -\frac{\vx_i}{2} - \frac{\vx_i^2}{4}\bfbeta +\frac{\vx_i^4}{48}\bfbeta^3-\frac{\vx_i^6}{480}\bfbeta^5.
\end{eqnarray*}
The derivative of the approximate log-likelihood function for the logit model is then given by
\begin{eqnarray*}
\ell^\prime(\bfbeta)&=&\sum_{i=1}^n\Bigg\{y_i\left(\frac{\vx_i}{2} - \frac{\vx_i^2}{4}\bfbeta +\frac{\vx_i^4}{48}\bfbeta^3-\frac{\vx_i^6}{480}\bfbeta^5\right)\Bigg\}\\
&+&\sum_{i=1}^n\Bigg\{(1-y_i)\left(-\frac{\vx_i}{2} - \frac{\vx_i^2}{4}\bfbeta +\frac{\vx_i^4}{48}\bfbeta^3-\frac{\vx_i^6}{480}\bfbeta^5\right)\Bigg\},
\end{eqnarray*}
which, upon ignoring the higher degree terms in the expansion becomes
$$
\ell^\prime(\bfbeta) \simeq \sum_{i=1}^n\Bigg\{4y_i\vx_i-2\vx_i-\vx_i^2\bfbeta\Bigg\}.
$$
It is straightforward to see that solving $\ell^\prime(\bfbeta)=0$ for $\bfbeta$ yields
$$
{\hat{\bfbeta}^{(\tt logit)}} \simeq 2\left[\frac{2\displaystyle\sum_{i=1}^n{\vx_i y_i} - \displaystyle\sum_{i=1}^n{\vx_i}}{\displaystyle \sum_{i=1}^n{\vx_i^2}}\right].
$$
If we now consider the probit link function, we have $\pi(\vx_i) = \Phi(\bfbeta\vx_i) = \int_{-\infty}^{\bfbeta \vx_i}{\frac{1}{\sqrt{2\pi}}e^{-\frac{1}{2}z^2}}$. Using a derivation similar to the one performed earlier, and ignoring higher order terms, we get
\begin{eqnarray*}
\ell^\prime(\bfbeta) = \frac{\partial {\ell(\bfbeta)}}{\partial \bfbeta} &\simeq&
\sum_{i=1}^n\Bigg\{{y_i}\left(c_1\vx_i - 2 c_2\bfbeta\vx_i^2\right)\Bigg\} \\
&+& \sum_{i=1}^n\Bigg\{({1-y_i})
\left(-c_1\vx_i - 2 c_2\bfbeta\vx_i^2\right)\Bigg\}\\
&=& \sum_{i=1}^n\Bigg\{2c_1\vx_i y_i -c_1\vx_i-2 c_2\bfbeta\vx_i^2\Bigg\}
\end{eqnarray*}
where $c_1=0.797885$ and $c_2=0.31831$. This leads to
$$
\hat{\bfbeta}^{(\tt probit)} \simeq \frac{c_1}{2 c_2}\left[\frac{2\displaystyle\sum_{i=1}^n{\vx_i y_i} - \displaystyle\sum_{i=1}^n{\vx_i}}{\displaystyle \sum_{i=1}^n{\vx_i^2}}\right].
$$
It is then straightforward to see that
$$
\frac{\hat{\bfbeta}^{(\tt probit)}}{\hat{\bfbeta}^{(\tt logit)}} \simeq \frac{c_1}{4 c_2} = 0.625,
$$
or equivalently
$$
\hat{\bfbeta}^{(\tt probit)} \simeq 0.625 \hat{\bfbeta}^{(\tt logit)}.
$$
\end{proof}

It must be emphasized that the above likelihood-based theoretical verifications are dependent on Taylor series approximations of the likelihood and therefore
the factor of proportionality are bound to be inexact. It's re-assuring however to see that our
computational verification  does confirm the results found by theoretical derivation.

\section{Similarities and Differences beyond Logit and Probit}
\noindent Other aspects of our work reveal that the similarities proved and demonstrated above between
  the probit and the logit link functions extend predictively to the other link functions mentioned above.
As far as structural equivalence or the lack thereof is concerned,
Appendix A contains similar derivations for the relationship between cauchit and logit, and the relationship between compit and logit.
As far as, predictive equivalence is concerned, we now present a verification based on the computation of many replications of the test error.

\subsection{Computational Verification of Predictive Equivalence}
We now computationally compare the predictive merits of each of the four link functions considered so far. To this end,
we compare the estimated average test error yielded by the four link functions.
We do so by running $R = 10000$ replications of the split of the data set into training and test set, and at each iteration we compute
the corresponding test error for the classifier corresponding to each link functions.
For one iteration/replication for instance,
$\hat{R}_{\tt \tiny test}(\hat{f}^{(\tt probit)})$, $\hat{R}_{\tt \tiny test}(\hat{f}^{(\tt compit)})$,
$\hat{R}_{\tt \tiny test}(\hat{f}^{(\tt cauchit)})$ and $\hat{R}_{\tt \tiny test}(\hat{f}^{(\tt logit)})$ are the values of the test error generated by probit, compit, cauchit and logit respectively. After $R$ replications, we have $R$ random realizations of each of those four test errors. We then perform various statistical calculations on the $R$ replications, namely {\em median, mean, standard deviation, kurtosis, skewness, IQR etc...}, to assess the similarity and the differences among the link functions. We perform the similar $R$ replications for model comparison using both AIC and BIC.

{\tt Example 4:} {\sf Verification of Predictive Equivalence on Artificial Data:}
$\{(\vx_i,y_i), i=1,\cdots,n\}$ where $\vx_i \sim {\tt Normal}(0,2^2)$ and $y_i \in \{0,1\}$ are drawn for a cauchy binary regression model with
$\beta_0=1$ and $\beta_1=2$, namely $Y_i \sim {\tt Bernoulli}(\pi(\vx_i))$ where
$$
 \pi(\vx_i) =\Pr[Y_i = 1 | \vx_i] = \frac{1}{\pi}\left[\tan^{-1}\left(1 + 2\vx_i\right)+\frac{\pi}{2}\right].
$$
Table\eqref{tab:test:artif:1} shows some statistics on $R=10000$ replications of the test error. The above results suggest that the four link functions
are almost indistinguishable as the estimated statistics are almost all equally across the examples.
\begin{table}[!htbp]
\centering  
\begin{tabular}{lrrrr}
  \toprule
     & {\tt probit} & {\tt compit} & {\tt cauchit} & {\tt logit} \\  \hline
  {\tt median}   &  0.16  &   0.16 &  0.16   &  0.16  \\
  {\tt mean}     &  0.16  &  0.16  &  0.16   &  0.16  \\
  {\tt sd}       &  0.04  &  0.04  &  0.03   &  0.04  \\
  {\tt skewness} &  0.21  &  0.26  &  0.26   &  0.24  \\
  {\tt kurtosis} &  3.18  &  3.51  &  3.20   &  3.20  \\
  {\tt cv}       & 22.56  & 22.46  & 22.25   & 22.57  \\
  {\tt IQR}      &  0.05  &  0.04  &  0.05   &  0.05  \\
  {\tt min}      &  0.06  &  0.04  &  0.06   &  0.06  \\
  {\tt max}      &  0.30  &  0.32  &  0.31   &  0.30  \\
  \bottomrule
\end{tabular}
\caption{\it Statistics based on $R=10000$ replicates of the test error on the artificial data set described above.
It's clear that the values are indistinguishable across the four link functions.} 
\label{tab:test:artif:1} 
\end{table}

\noindent  {\tt Example 5:} {\sf Verification of Predictive Equivalence on the Pima Indian Diabetes Dataset:}
We once again consider the famous Pima Indian Diabetes
dataset. The Pima Indian Diabetes Dataset is arguably one the most used benchmark data sets in the statistics and pattern recognition community. As can be see in Table \eqref{tab:test:pima:1}, there is virtually no difference between the models. In other words, on the Pima Indian Diabetes data set, the four link functions are predictive equivalent.
\begin{table}[!htbp]
\centering  
\begin{tabular}{lrrrr}
  \toprule
      & {\tt probit} & {\tt compit} & {\tt cauchit} & {\tt logit} \\  \hline
  {\tt median}   &  0.25  &  0.24   &  0.25  &  0.25  \\
  {\tt mean}     &  0.25  &  0.25   &  0.26  &  0.25  \\
  {\tt sd}       &  0.04  &  0.04   &  0.05  &  0.04  \\
  {\tt skewness} &  0.06  &  0.07   &  0.06  &  0.06  \\
  {\tt kurtosis} &  2.92  &  2.95   &  2.95  &  2.92  \\
  {\tt cv}       & 17.84  & 18.33   & 17.62  & 17.85  \\
  {\tt IQR}      &  0.06  &  0.06   &  0.07  &  0.06  \\
  {\tt min}      &  0.09  &  0.07   &  0.10  &  0.09  \\
  {\tt max}      &  0.43  &  0.40   &  0.45  &  0.42  \\
  \bottomrule
\end{tabular}
\caption{\it Statistics based on $R=10000$ replicates of the test error on the Pima Indian Diabetes data set.
It's quite obvious that the values are indistinguishable across the four link functions.}
\label{tab:test:pima:1}
\end{table}

\noindent It's also noteworthy to point out that all the four models also yield similar goodness of fit measures when scored using AIC and BIC.
Indeed, Figure (\ref{fig:aic:bic:pima:1}) reveals that over the $R=10000$ replications of the split of the data into training and test set,
both the AIC and BIC are distributionally similar across all the four link functions.
\begin{figure}[!htbp]
\centering
\begin{tabular}{lr}
\epsfig{figure=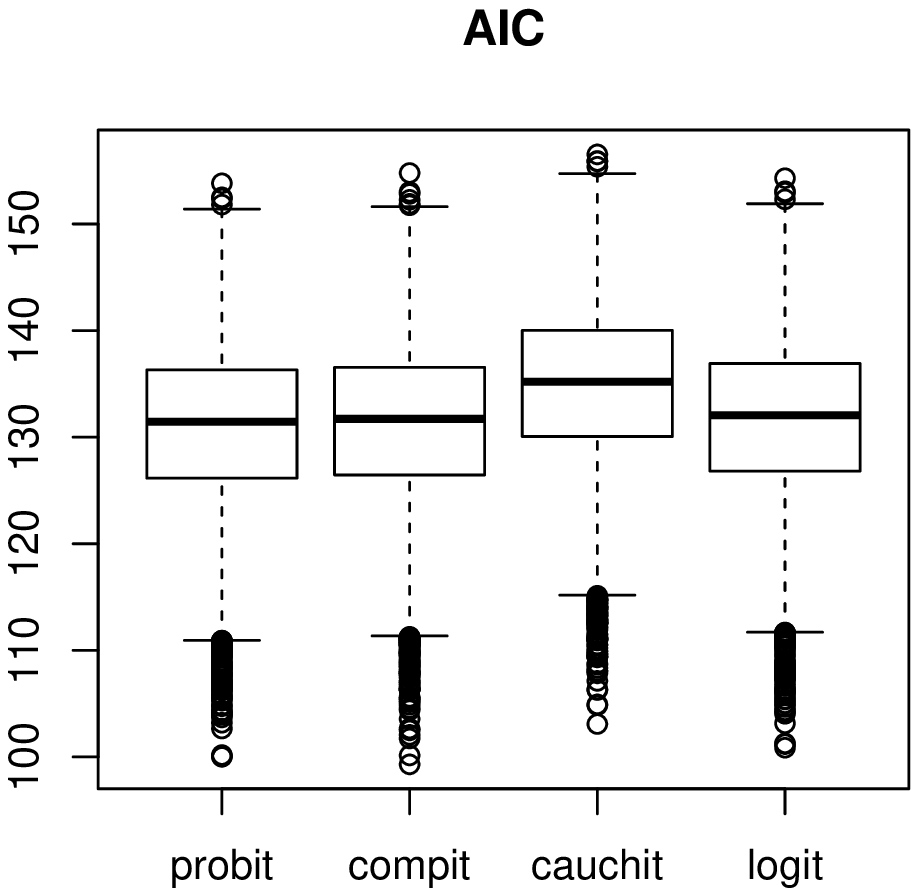, height=8cm, width=8cm}&
\epsfig{figure=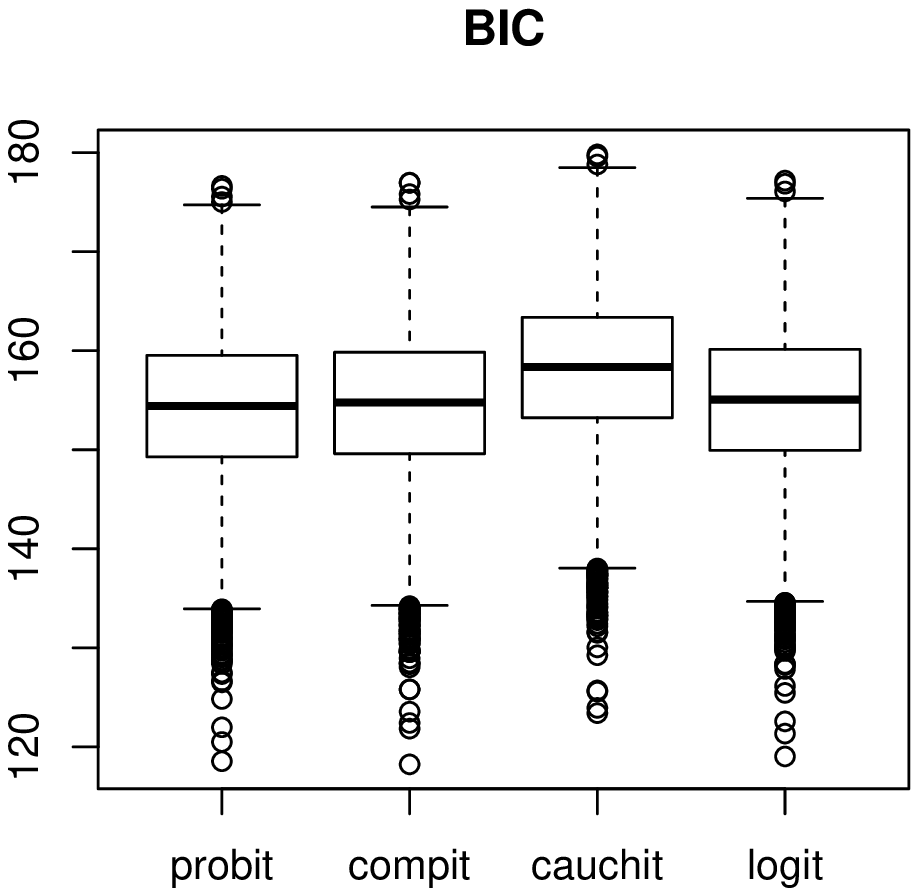, height=8cm, width=8cm}
\end{tabular}
\caption{{\it Comparative boxplots of both AIC and BIC across the four link functions based on $R=10000$ replications of model fittings.}}
\label{fig:aic:bic:pima:1}
\end{figure}
\noindent Despite the slight difference shown by the Cauchit model, it is fair to say that all the link functions are
equivalent in terms of goodness of fit. Once again, this is yet another evidence to support and somewhat reinforce/confirm \cite{Feller}'s
claim that all these link functions are equivalent in terms of goodness of fit, and that the over-glorification of the logit model
is at best misguided if not unfounded.

\subsection{Evidence of  Differences in High Dimensional Spaces}
{\sf Simulated evidence: } We generate $s=10000$ observations in the interval $[-15,15]$. For each
link function, we compute the sign of $F(\vx_i)-1/2$ for $i=1,\cdots, s$. We then generate a table containing
the percentage of times the signs differ.
\begin{table}[!htbp]
\centering  
\begin{tabular}{lrrrr}
  \toprule
      & {\tt probit} & {\tt compit} & {\tt cauchit} & {\tt logit} \\  \hline
  {\tt probit}   &  0.000  &  0.004   &  0.000  &  0.000  \\
  {\tt compit}   &  0.004  &  0.000   &  0.004  &  0.004  \\
  {\tt cauchit}  &  0.000  &  0.004   &  0.000  &  0.000  \\
  {\tt logit}    &  0.000  &  0.000   &  0.000  &  0.000  \\
  \bottomrule
\end{tabular}
\caption{\it All the pairs reveal a disagreement of $0\%$ except the pairs involving the compit.}
\label{tab:link:disagree}
\end{table}

\noindent {\sf Computational Demonstrations on the Email Spam Data: }
Unlike all the other data sets encountered thus far, the email spam data set is
a fairly high dimensional data set. It has a total of $p=57$ variables and
$n=4601$ observations.
\begin{table}[!h]
\centering  
\begin{tabular}{lrrrr}
  \toprule
       & {\tt probit} & {\tt compit} & {\tt cauchit} & {\tt logit} \\  \hline
  {\tt median}   &  0.08  &  0.13   &  0.07  &  0.07  \\
  {\tt mean}     &  0.10  &  0.13   &  0.07  &  0.08  \\
  {\tt sd}       &  0.03  &  0.03   &  0.04  &  0.01  \\
  {\tt skewness} &  1.75  &  0.88   &  9.16  &  4.86  \\
  {\tt kurtosis} &  9.29  &  8.56   &  103.95  & 40.58  \\
  {\tt cv}       & 34.41  & 20.61   & 51.15  & 14.32  \\
  {\tt IQR}      &  0.04  &  0.04   &  0.01  &  0.01  \\
  {\tt min}      &  0.06  &  0.07   &  0.05  &  0.06  \\
  {\tt max}      &  0.41  &  0.38   &  0.62  &  0.18 \\
  \bottomrule
\end{tabular}
\caption{Email Spam Data Set Results}
\label{tab:spam:1}
\end{table}

\noindent Clearly, the results depicted in Table \eqref{tab:spam:1} reveal some drastic differences
in performance among the four link functions on this rather high dimensional data.
The boxplots below reinforce these findings as they show that in terms of goodness of fit measured through
AIC and BIC, the compit model deviates substantially from the other models.
\begin{figure}[!htbp]
\centering
\begin{tabular}{lr}
\epsfig{figure=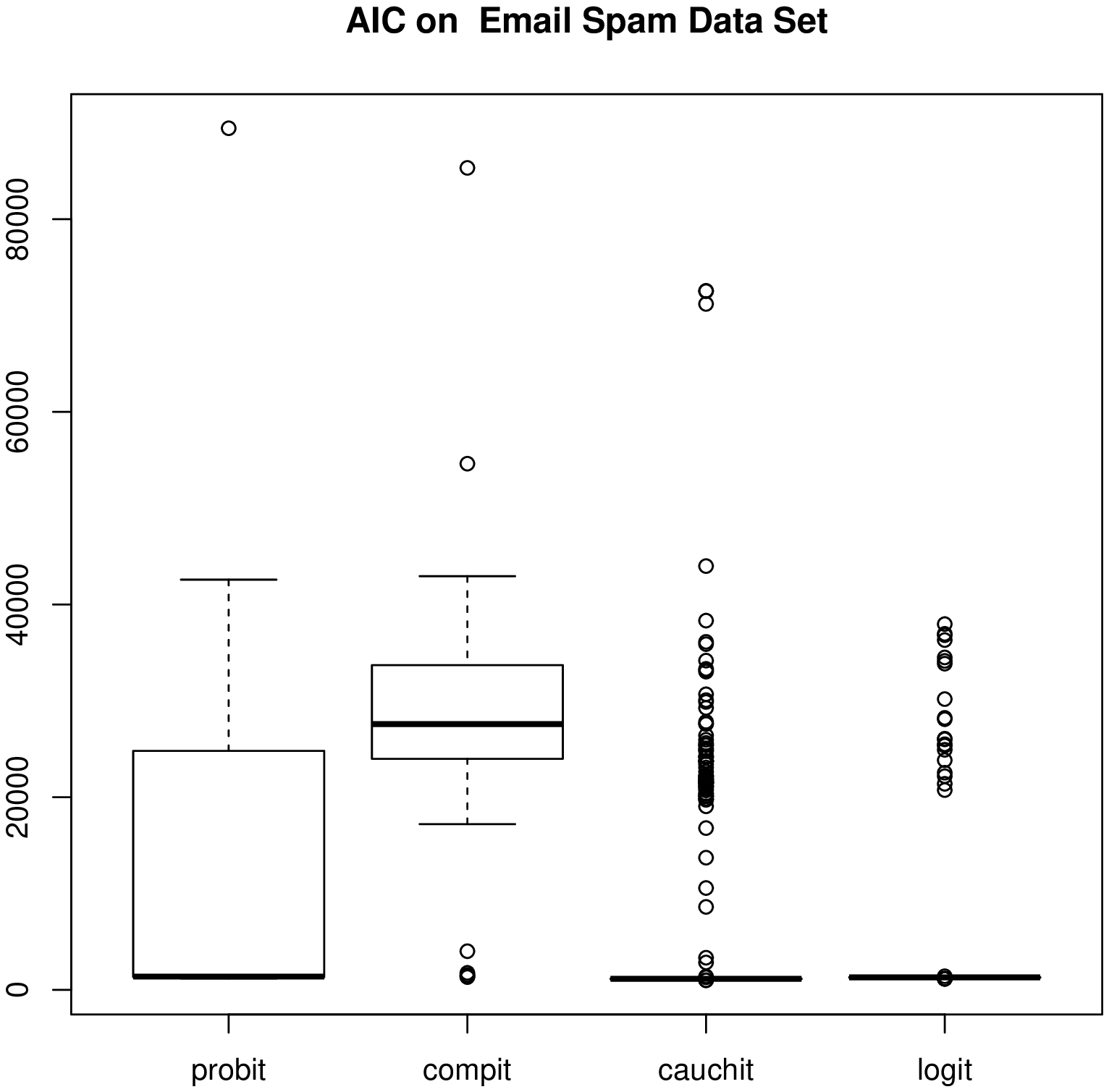, height=5.0cm, width=7cm}&
\epsfig{figure=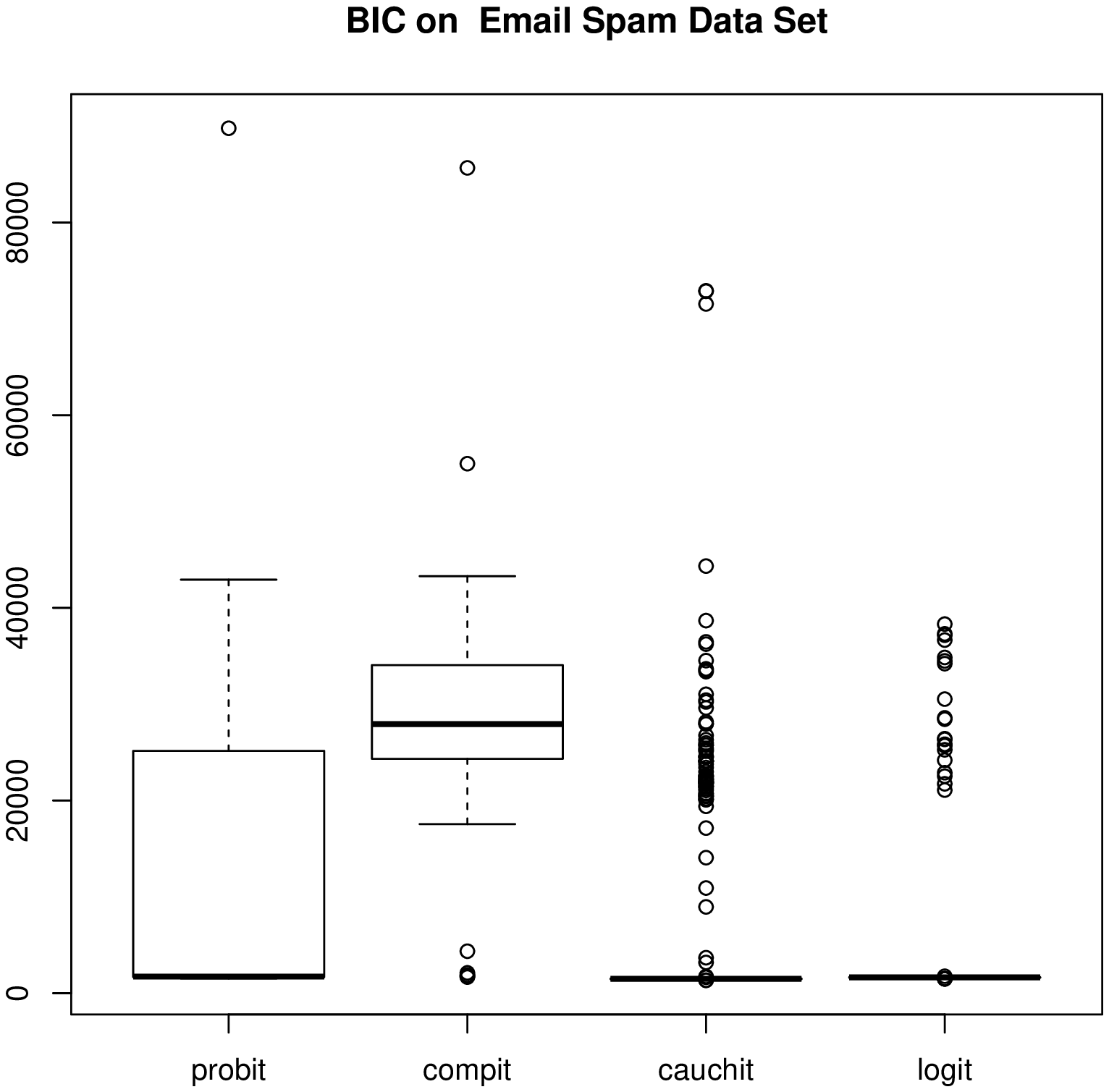, height=5.0cm, width=7cm}
\end{tabular}
\caption{{\it Comparative Boxplots assessing the goodness of fit of the four link functions using AIC and BIC over $R=10000$
replications of model fitting under each of the link functions.}}
\label{fig:aic-bic-2}
\end{figure}

\section{Conclusion and discussion}
Throughout this  paper, we have explored both conceptually/methodologically  and
computationally the similarities  among four of the most commonly used link
functions in binary regression. We have theoretically shed some light on some of the structural
reasons that explain the indistinguishability in performance in the univariate settings
among the four link functions considered. Although section 2 concentrated mainly on the equivalence of
the logit and probit, the Appendix provides a similar derivation for both the
cauchit and the complementary log log link functions. We have also demonstrated by computational
simulations that the four link functions are essentially equivalent both structurally and predictively
in the univariate setting and in  low dimensional spaces. Our last example showed computationally that the four link functions might differ quite substantially when the dimensional of the input space becomes extremely large. We notice specifically that the performance in high dimensional spaces tends to defend on the internal structure of the input: completely orthogonal designs tending to bode well with
all the perfectly symmetric link functions while the non orthogonal designs deliver best performances
under the complementary log log. Finally, the sparseness of the input space tends to dictate the
choice of the most appropriate link function, Cauchit tending to be the model of choice under
high level of sparseness. In our future work, we intend to provide as complete a theoretical characterization as possible in extremely high dimensional
spaces, namely providing the conditions under which each of the link function will yield the best fit for the data.

\bibliographystyle{chicago}
\bibliography{ernest-necla-1}

\section{Appendix A}
\label{appendix:1}
\begin{theorem}Consider an i.i.d sample $(\vx_1,y_1),(\vx_2,y_2),\cdots,(\vx_n,y_n)$ where  $\vx_i\in \mathbb{R}$ is a real-valued predictor variable, and  $y_i \in \{0,1\}$ is the corresponding binary response. First consider fitting the cauchit model
$\Pr[Y_i = 1 | \vx_i] = \pi(\vx_i) = \frac{1}{\pi}\left[\tan^{-1}(\bfbeta \vx_i)+\frac{\pi}{2}\right]$ to the data, and let $\hat{\bfbeta}^{(\tt cauchit)}$ denote the corresponding estimate of $\bfbeta$. Then consider fitting the logit model and $\Pr[Y_i = 1 | \vx_i] = \pi(\vx_i) = 1/(1+\exp(-\bfbeta \vx_i))$ to the data, and let $\hat{\bfbeta}^{(\tt logit)}$ denote the corresponding estimate of $\bfbeta$. Then,
$$
\hat{\bfbeta}^{(\tt cauchit)} \simeq \frac{\pi}{4} \hat{\bfbeta}^{(\tt logit)}.
$$
\end{theorem}

\begin{proof}
Given an i.i.d sample $(\vx_1,y_1),(\vx_2,y_2),\cdots,(\vx_n,y_n)$ and the model $\Pr[Y_i = 1 | \vx_i] = \pi(\vx_i)$,  the loglikelihood for $\bfbeta$ is given by
\begin{eqnarray}
\ell(\bfbeta) = \log L(\bfbeta) = \sum_{i=1}^n\Big\{{y_i}\log\pi(\vx_i) + ({1-y_i})\log(1-\pi(\vx_i))\Big\}.
\end{eqnarray}
For the Cauchit for instance, $\pi(\vx_i) = \frac{1}{2} + \frac{1}{\pi}\tan^{-1}(\bfbeta \vx_i)$.
We use the Taylor series expansion around zero for both $\log(\pi(\vx_i))$ and $\log(1-\pi(\vx_i))$.
$$
\log \pi(\vx_i) = -\log 2 + \frac{2\bfbeta\vx_i}{\pi} - \frac{2\bfbeta^2\vx_i^2}{\pi^2}
- \frac{2(\pi^2-4)\bfbeta^3\vx_i^3}{3\pi^3} + O(\vx_i^4)
$$
and
$$
\log(1-\pi(\vx_i)) = -\log 2 - \frac{2\bfbeta\vx_i}{\pi} - \frac{2\bfbeta^2\vx_i^2}{\pi^2}
+ \frac{2(\pi^2-4)\bfbeta^3\vx_i^3}{3\pi^3} + O(\vx_i^4)
$$

A first order approximation of the derivative of the log-likelihood with respect to $\bfbeta$ is
\begin{eqnarray*}
\ell^\prime(\bfbeta) = \frac{\partial {\ell(\bfbeta)}}{\partial \bfbeta} &=&
\sum_{i=1}^n\Bigg\{{y_i}\left(\frac{2\vx_i}{\pi} - \frac{4\bfbeta\vx_i^2}{\pi^2}\right) + ({1-y_i})
\left(-\frac{2\vx_i}{\pi} - \frac{4\bfbeta\vx_i^2}{\pi^2}\right)\Bigg\}\\
&=& \sum_{i=1}^n\Bigg\{\frac{4}{\pi}\vx_i y_i -\frac{2}{\pi}\vx_i-\frac{4}{\pi^2}\bfbeta\vx_i^2\Bigg\}
\end{eqnarray*}
Solving $\ell^\prime(\bfbeta) =0$ yields
$$
\hat{\bfbeta} = \frac{\frac{4}{\pi}\sum_{i=1}^n{\vx_i y_i} - \frac{2}{\pi}\sum_{i=1}^n{\vx_i}}{\frac{4}{\pi^2}\sum_{i=1}^n{\vx_i^2}}
$$
which simplifies to
$$
\hat{\bfbeta}^{(\tt cauchit)} = \frac{\pi}{2}\left[\frac{2\displaystyle\sum_{i=1}^n{\vx_i y_i} - \displaystyle\sum_{i=1}^n{\vx_i}}{\displaystyle \sum_{i=1}^n{\vx_i^2}}\right]
$$
\end{proof}

\end{document}